\documentclass[11pt, fleqn, a4paper, oneside]{imsart}

\usepackage{amsthm,amsmath,amssymb}
\usepackage{graphicx}

\usepackage{bbm}

\usepackage[authoryear]{natbib}
\usepackage{algorithm}
\usepackage{algorithmic}

\usepackage[margin=1in]{geometry}

\bibpunct{(}{)}{;}{a}{,}{,}
\RequirePackage[OT1]{fontenc}
\startlocaldefs

\newtheorem{theorem}{Theorem}

\newtheorem{lemma}{Lemma}
\newtheorem{proposition}{Proposition}
\newtheorem{example}[theorem]{Example}

\theoremstyle{remark}

\theoremstyle{definition}
\newtheorem{definition}{Definition}

\newtheorem{assum}[theorem]{Assumption}

\newcommand{\set}[1]{\left\lbrace#1\right\rbrace}

\DeclareMathOperator{\Var}{Var}

\endlocaldefs

\begin{document}

\begin{frontmatter}

%% "Title of the paper"
%\title{Unsupervised nonparametric detection of unknown objects in noisy images based on percolation theory}
%\runtitle{Unsupervised detection and percolation}

\title{Unsupervised robust nonparametric learning of hidden community properties}
\runtitle{Learning hidden community properties}

\begin{aug}
\author{\snm{Mikhail} \fnms{Langovoy}\corref{}\thanksref{t2}
\ead[label=e1]{mikhail.langovoy@epfl.ch}}

\affiliation{EPFL, Switzerland}

\address{Machine Learning and Optimization Laboratory\\
       EPFL, Station 14 \\
       Lausanne, CH-1015 Switzerland\\
\printead{e1}}

\and

\author{\snm{Akhilesh} \fnms{Gotmare} \ead[label=e2]{akhilesh.gotmare@epfl.ch}}

\affiliation{EPFL, Switzerland}

%\address{Lehrstuhl A f\"{u}r Mathematik \\ RWTH Aachen, 52056 Aachen \\
%\printead{e2}}

\address{\printead{e2}}

\and

\author{\snm{Martin} \fnms{Jaggi} \ead[label=e3]{martin.jaggi@epfl.ch}}

\affiliation{EPFL, Switzerland}

%\address{Department of Statistics, \\
%University of California at Davis, Davis CA, \\
%95616-8572, USA\\
%\printead{e3}}

\address{\printead{e3}}

\thankstext{t2}{Corresponding author.}

%\runauthor{M. Langovoy, O. Wittich and P. L. Davies}
\runauthor{M. Langovoy et al.}
\end{aug}

\smallskip

%\begin{abstract}
%We propose a novel statistical method for detection of objects in noisy images. The method uses results from percolation and random graph theories. We present an algorithm that allows to detect objects of unknown shapes in the presence of nonparametric noise of unknown level. The noise density is assumed to be unknown and can be very irregular. Our procedure substantially differs from wavelets-based algorithms. The algorithm has linear complexity and exponential accuracy and is appropriate for real-time systems. We prove results on consistency and algorithmic complexity of our procedure.\\
%\end{abstract}

\begin{abstract}
  We consider learning of fundamental properties of communities in large noisy networks, in the prototypical situation where the nodes or users are split into two classes according to a binary property, e.g., according to their opinions or preferences on a topic. For learning these properties, we propose a nonparametric, unsupervised, and scalable graph scan procedure that is, in addition, robust against a class of powerful adversaries. In our setup, one of the communities can fall under the influence of a knowledgeable adversarial leader, who knows the full network structure, has unlimited computational resources and can completely foresee our planned actions on the network. We prove strong consistency of our results in this setup with minimal assumptions. In particular, the learning procedure estimates the baseline activity of normal users asymptotically correctly with probability 1; the only assumption being the existence of a single implicit community of asymptotically negligible logarithmic size. We provide experiments on real and synthetic data to illustrate the performance of our method, including examples with adversaries.
\end{abstract}


\end{frontmatter}

\section{Introduction}\label{Section_Intro}

We develop robust and scalable methods to uncover global properties of communities hidden in large networks with noisy node attributes. Consider the fundamental situation where the nodes or users in the network are split into two classes according to a binary property, such as their opinion or preferences on a specific topic. We call these two classes the ``active'' and ``inactive'' users, respectively. Examples include support of a particular candidate in elections~\citep{adamic2005political}, or a level of interest in a particular topic, or a degree of support of certain statement. Pixels in digital images can be viewed as network nodes with attributes, so many applications from image processing, such as road tracking \citep{Geman_1996_Roads} or medical tumor detection \citep{Mcinerney_1996_Medical_Image}, can be treated within this framework as well. In disease outbreak detection in epidemiology, people are naturally falling into two categories, healthy or sick \citep{Rotz_Huges_2004}.

Additionally, motivated by modern real-world settings, we assume that the network of interest is too large to be processed manually, especially for each possible topic of interest. Therefore, activity observations of users are determined and delivered to us by a third-party algorithm called the \emph{crawler}. Naturally, the crawler has its classification and learning errors that are not known to us. Therefore, we treat a general non-parametric case of the crawler error probabilities.

Our general goal is to learn global properties of communities of active and inactive users despite such noise and errors, in an unsupervised way, while additionally being robust to an adversary. The specific goal of this paper is unsupervised learning of the infinite-dimensional vector $(a, b, F)$, where $a$ and $b$ are baseline activities corresponding to inactive and active users correspondingly. They are fundamental global intrinsic properties of topical communities of users and are needed to be able to detect and recover active and inactive communities in subsequent analysis. An unknown probability distribution function $F$ represents performance of the crawling algorithm for the particular topic, and serves as a natural measure of quality of our knowledge about the network's content.

\subsection{Distinctive features of our framework}
We infer global characteristic properties of communities, using both the network topology as well as the network's content represented by node attributes (see, e.g., \citep{Ruan_2013_Content_Links} or \citep{Yang_2013_Community_Node} for related types of setups). Regarding the topology, we assume that the edges can be observed, but do not impose any global assumptions on the graph structure or any edge density assumptions, be it global or local. Naturally, opposing communities in our setup do not overlap.

%Notably,

As for node attributes, we assume that the nodes are corrupted by nonparametric noise of unknown and nonsmooth distribution. Notably, our noise model permits strong and long-distance correlations between observations on active vertices.

%We also do not assume binary noise (often used for community recovery), and do not assume Gaussian noise (engineering literature, and most of statistical literature on community detection and recovery), and do not ask for a noise distribution to be known (so that community detection in this setup would be a composite nonparametric hypothesis testing problem, for example).

%allows for observations where the active vertices need not form completely, nor be pairwise independent collections of random variables; the model also

We treat the setup where \emph{active} users can fall under the influence of an \emph{adversary} who is capable of directly altering their activity values, pursuing the goal of spoiling our uncovering of community and network properties. The adversary knows the true values on active vertices and which vertices are inactive, knows the full graph structure, and has unlimited computational power and memory. Moreover, the adversary has a special deal with \emph{hackers} and can completely foresee the actions we will be doing on the network, including the outcomes of all our randomized procedures, in case if we choose to use any random or pseudo-random number generator.
\medskip

\noindent In summary, our key \textbf{contributions} are:
\begin{list}{–}{\leftmargin=1em}
\setlength{\itemsep}{10pt}
\item A robust scalable graph $k$-NN scan estimator (\S\ref{sec:knn}). The $k$-NN scans  presented here generalize sliding windows, moving averages, and scan statistics, extending them to the case of general graphs. Our probabilistic analysis of the estimators is related to extreme order statistics for dependent random variables.
\item Sufficient conditions for $k$-NN graph scan estimators to be consistent for learning global properties of communities in an adversarial framework. This is a remarkable property of these local scans, as it can be easily seen that most existing methods can easily be spoiled either by such an adversary or by sheer dimensionality of a very general nonparametric framework that we are considering.
\item Discussion of aspects of our estimator's computation to allow large scale graphs, in particular via a highly decentralized implementation, allowing scalable distributed and parallel graph processing.
\end{list}
\medskip

Our estimator is non-parametric, yet scalable method for learning properties of communities in noisy graphs. Moreover, it operates under minimal assumptions about the structure of the network and the communities. Unlike many existing methods, it utilizes the content of the network and thus does not require any global assumptions on the topology of the network or on relations between communities, such as graph sparsity or presence of hidden highly connected communities. The only local condition that we impose on the network is that it has one compact locally connected community of size that grows at least as a logarithm of the total number of vertices; this community can be asymptotically negligible in the large scale setting. Therefore, we believe that these results are foundational to fast consistent algorithms for community detection via percolation on general graphs. For the special case of lattices, the automated detection theorem \citep{langovoy_habeck_schoelkopf_JSM} serves as an example illustration of this approach.

%This permits to treat general graphs without limiting attention to lattices, or Erd\"os-Reny\'i random graphs, or preferential attachment structures.

Finally, learning the actual performance of the crawler can also be addressed as a by-product of results of this paper. We note that our results are, of course, valid in their present form in the special case when there is no adversarial action.

\subsection{Related work}

Scan statistics have been long used for detection of unexpected events and for nonparametric estimation. The initial idea and the first development of the underlying theory, for one-dimensional discrete case, goes back at least to \citep{uspensky1937introduction}, who studied longest runs of successes in Bernoulli sequences. In the two-dimensional case, a surge of interest to continuous scan statistics has been sparked by \citep{kulldorff1997}. This and related types of Euclidean spatial scan statistics found numerous applications in geostatistics, medicine, epidemiology and ecological studies, see examples in \citep{tango2005}, \citep{kulldorff1999,kulldorff2006}, \citep{patil2004}. The scan statistic methodology is compatible with Bayesian paradigm as well \citep{neill2006}. Notably, despite substantial efforts, most of the work in this area is based on heuristics and experimentation rather than on rigorous probabilistic analysis.

Recently, a new line of research emerged where discrete versions of multidimensional scan statistics were applied to discrete structures such as pixelized images or lattices \citep{langovoy_report_2009-035}, \citep{Langovoy_Wittich_Square}, \citep{langovoy_wittich_robust}, \citep{Arias-Castro_Grimmett}. The idea of discrete scans proved useful in application areas like anomaly detection and automated detection of unknown objects in extremely noisy images \citep{langovoy_habeck_schoelkopf_JSM}, \citep{langovoy_habeck_schoelkopf}. Surprisingly, discrete scan statistics were rigorously analyzed by means of random graph and percolation theories \citep{Gri:99}.

As a natural extension of this idea, several variations of scan statistics for graphs were proposed. Examples include non-parametric scan statistics for event detection and forecasting in heterogeneous social media graphs \citep{chen2014}, changepoint detection over graphs with the spectral scan statistic \citep{sharpnack2013}, anomaly detection in graphs \citep{sharpnack2013b}, graph topic scan statistic for spatial event detection \citep{liu2016}.

In the present paper, we extend the previous research in this area by constructing a fully nonparametric unsupervised learning graph scan-based procedure that is robust against an extremely malicious adversary. Moreover, we prove consistency of our results in a framework with minimal assumptions.

The outline of this paper is as follows. In Section \ref{sec.formulation}, we introduce the framework for analyzing noisy or indirectly observed graphs with users polarized into two types of communities according to a binary property, and introduce our concept of an adversarial leader influencing communities of one of these types. In Section \ref{sec:knn}, we introduce the concept of $k$-NN graph scan estimators and define the main notions related to their construction. These estimators are one of the main ingredients for robust consistent inference in this paper. In Section \ref{sec:consistency}, we establish consistency and derive important properties of these unsupervised estimators. Proofs of all the results can be found in Appendix. Experimental results for real and synthetic graphs, as well as for the cases with and without adversaries can be found in Section \ref{Section_Experiments}. Scalability of the method is established and the algorithm's distributed realization is discussed in Section \ref{Section_Optimization}.

%their opinion on a particular matter

\section{Formulation: model and adversaries} %
\label{sec.formulation}
Let $G_n = (V_n, E)$ be a graph with $n$ vertices and $|E|$ edges. We do not impose restrictions on $G_n$ such as sparsity, so $|E| \,\gg\, n$ is possible. For any vertex $v \in V$ we observe a real-valued random variable $X_v : \Omega \rightarrow \mathbb{R}$, defined on an appropriate probability space $\Omega$.

%We call the collection $\{ X_v \,|\, v \in V_n \}$ a \emph{random coloring} of the graph $G_n$.\footnote{We can also consider infinite colorings based on $X_v : \Omega \rightarrow \mathbb{N}$; but for simplicity we study only finite ones.}
We call the random variables $X_v$ the \emph{observed activities}. These observations are noisy realizations of the \emph{true activity} $A_v$ (of each vertex $v$), which we only observe with an additive (\emph{nonparametric})  noise. More precisely,
\begin{equation}\label{1}
  X_v = A_v +  \varepsilon_{v}\,,   %MJ: i would use a different notation of Im. can make a macro also if needed
\end{equation}
where we assume that the noise $\{ \varepsilon_{v}\}$ satisfies
\begin{equation}\label{2}
  \varepsilon_v\sim F,\quad\mathbb{E}[\varepsilon_{v}] = 0,  \quad \Var[\varepsilon_{v}] = \sigma^{2} < + \infty.
\end{equation}
We \emph{do not} assume knowledge of either the distribution~$F$ or the variance $\sigma^2$. Moreover, $F$ need not have any particular parametric form, or be continuous, or be limited to being discrete with finite or countable support.

Given this model, our general aim is to robustly learn, without supervision, as much as possible about the collection of the true underlying activities $\set{A_v}_{v\in V}$ from noisy observations~\eqref{1}; and possibly, to also learn $F$ nonparametrically. Moreover, we admit presence of a powerful adversary who can corrupt the graph to hinder inference (\S\ref{sec:adv}).

The specific problem solved in this paper can be formulated as unsupervised learning of the infinite-dimensional vector $(a, b, F)$, where $a$ and $b$ are baseline activities corresponding to inactive and active users correspondingly. They are fundamental global intrinsic properties of topical communities of users and are needed to be able to detect and recover active and inactive communities in subsequent analysis. An unknown probability distribution function $F$ represents performance of the crawling algorithm for the particular topic.

%In this model, $A_v$ denotes the (unknown) value on the vertex $v$. Intuitively, we can think of $A_v$ as the level of activity or support to a certain topic that $v$ intends to reveal to an observer. This level is observed imprecisely due to noise incurred by our network crawling algorithm.

%In thiWe will learn $F$ nonparametrically. %The influence of the adversary will be clarified shortly.

% Before refining our description of the model~\eqref{1}-\eqref{2} further, let us remark captures several realistic scenarios.
\begin{example}
  Consider a large-scale network where manual processing is not possible and the activities of users (nodes) are determined using algorithms that may either be operating on outdated data or have large classification errors. For instance, say in a social network the observed activity $X_v$ measures the level of support of a statement by user $v$ to a specified topic. It might be difficult for an algorithm to realize whether a statement ``Topic X is crazy'' is supportive of the topic or the opposite.
\end{example}

When estimating baseline activities within communities of a-priori unknown configurations and in the presence of nonparametric noise, it is important to assume that there exists an observability threshold. Thus, we assume that:
\begin{equation}\label{3}
  \text{for each}\ v \in V:\
  \begin{cases}
    A_v \geq b > a, & \text{if $v$ is active;} \\
    A_v = a, & \text{if $v$ is inactive.}
  \end{cases}
\end{equation}
Vertices with activity level $a$ are normal (inactive). Those with levels above the \emph{unknown threshold} $b$ are active.

%and one of the goals of this work is to suggest an estimator of their baseline activity.

This setup with two distinctly different types of communities is typical for applications where users have to be considered in terms of their opinion or preference on an important polarizing topic. Examples include political elections with two candidates, or important issues such as a healthcare reform, or good or bad health status of individuals in epidemiological applications and cells in medical images, etc.

We explicitly remark that we do assume two types of communities, but not the existence of two large communities covering the whole network. In fact, our setup includes the case when there is a large number of small communities of each type. The results of this paper allow strongly consistent inference for the cases when the number of communities of each type can be $O (n \, \log^{-1} n)$.

We specifically note here that the problem of estimating global baseline activity levels of communities is not the same as the problem of identifying active and inactive vertices in a network. For each particular vertex, the latter problem amounts to classification (or to clustering, in the unsupervised scenario), while the former problem is an infinite-dimensional learning problem. Moreover, as follows from results in this paper, global properties of hidden communities can be learned (in a strongly consistent way) without guessing activity levels of individual vertices. On contrary, it can be easily seen that, in the general setting of nonparametric noisy graphs, it is impossible to guarantee a strong form of consistency for uncovering true activity levels of individual vertices, and it is impossible to even consistently classify any individual vertex as an active or inactive one.

%\startblu
The lower bound for active vertex intensity $b$ is assumed to be unknown, and we propose an unsupervised procedure for learning $b$ as well, even though we do not analyze its performance in this paper. The difficulty of these estimation problems is that locations, shapes and exact sizes of clusters of active users are assumed to be unknown. The number of active users or active clusters is also unknown, and we make no probabilistic assumptions about this number or about the distribution of cluster locations. There can be anywhere between $O(\log n)$ and $n$ inactive users, and between $1$ and $O (n \, \log^{-1} n)$ inactive communities. The contribution of inactive (and active) users and communities can range from negligible to dominant.

It is important to note that we consider the case of a fully nonparametric noise of unknown level and having an unknown distribution; even within the setup with independent identically distributed and bounded noise, this model is far more general than traditional models with normally distributed errors and graphs with simple regular structure or parametric types of degree distributions.

\subsection{Adversarial leader}
\label{sec:adv}
Our model and inference algorithms permit presence of an adversary. We assume that \emph{active} vertices may be under the influence of an \emph{adversarial leader} who is capable of directly \emph{altering their values} $A_v$, pursuing the goal of spoiling our inference. The adversary knows the true values on active vertices and which vertices are inactive, knows the full graph structure of $G_n$, and has unlimited computational power and memory. Moreover, the adversary has a special deal with \emph{hackers} and can completely foresee the actions we will be performing on the network, including the outcomes of all our randomized procedures.

It is assumed that the adversary can only command its supporters, and therefore does not influence inactive vertices. This is a natural assumption that is met, for example, by political parties and their leaders, or by communities of bloggers under the influence of a particular opinion leader. The leader of a political party can command his supporters in a variety of ways and also can gain access to hidden information about the leader's supporters; meanwhile, all this information will be completely unaccessible for our inference.

It is also assumed that the adversary cannot completely tone down the active nodes, so that a separation  between (active) supporters and (inactive) non-supporters is maintained and condition \eqref{3} still holds, even though the separation rule is not known. Moreover, the adversary does not know and cannot influence the outcomes of the crawling algorithm, so that $\{ \varepsilon_v \,|\, v \in V_n \}$ is a completely independent collection of identically distributed random variables. These assumptions are met in many applications such as, for example, analysis of the blogosphere \citep{adamic2005political}, as the crawler is forming its evaluations using the history of users' activities; therefore, past activities of the user are accounted for and do not let the user completely mask his activity level.

%These assumptions are met in a number of applications. For example, when types of communities correspond to members of political parties supporting particular candidates in the elections, the leader of one of the parties can command his supporters in a variety of ways and also can gain access to hidden information about the leader's supporters; meanwhile, all this information will be completely unaccessible for our inference.

Without much loss of generality, we assume that the adversary influences the active vertices only once, before we run our inference, and cannot act at later steps. This assumption is not restrictive and is used to have more compact statements of consistency theorems. Indeed, we are considering an effectively static network (even if it is a particular observation of a dynamic network or of a random network); as the adversary knows everything we would do on the network in advance, his single action can be a combined response incorporating series of individual responses to any sequence of our actions. We illustrate this point in our experiments in Section \ref{Section_Experiments}, where we allow the adversary that is as strong as above and can act in multiple (even unlimited number of) steps, and we are still able to demonstrate empirical consistency of our method.

%For clarity, it is assumed that the adversary does not know the outcomes of the crawling algorithm.

%The adversary can only command its supporters, and therefore does not influence inactive vertices. This is a natural assumption that is met, for example, by political parties and their leaders, or by communities of bloggers under the influence of a particular opinion leader.

%The adversary has some limitations, though (otherwise, clearly, no consistent inference would be possible for us). It influences the active vertices only once, before we run our inference, and cannot act at any later steps. This is an assumption that we make to derive strong consistency theorems; the assumption is reasonable, as we are considering an effectively static network (even if it is a particular realization of the random network), and we ourselves can only act on the network once and without any supervision or prior knowledge.

Within this adversarial framework, we present below (Sec.~\ref{sec:knn}) new $k$-nearest neighbors based graph scan estimators and establish sufficient conditions for them to be consistent. This consistency is a strong type of robustness, and a particular strength of our approach, as it can be easily seen that many estimators can be spoiled by such an adversary, as the following example suggests.

Indeed, suppose we are using an ingenious method that allows us to select a ``nice'' set of local subgraphs on which we run some consistent estimator. The adversary, knowing our method fully, can increase the values on the few active vertices contained in the selected local subgraphs, thus either skewing our averaging on those subgraphs, or knocking off good subgraphs in case model selection is involved.

Of course, all the consistency results of the present paper are still true in the neutral case when there is no adversary spoiling our inference.

\section{$k$-NN graph scan estimators}\label{sec:knn}

%The goal of this paper can now be formulated as unsupervised learning of the infinite-dimensional vector $(a, b, F)$, where $a$ and $b$ are baseline activities corresponding to inactive and active users correspondingly. They are fundamental global intrinsic properties of topical communities of users and are needed to be able to detect and recover active and inactive communities in subsequent analysis. An unknown probability distribution function $F$ represents performance of the crawling algorithm for the particular topic.

%More specifically,

The true activity level $a$ corresponds to normal (inactive) vertices, while active vertices have the true activity level at least $b$. Since our estimators will be based on a scan by $k$-nearest-neighbors, we call these estimators $k$-NN \emph{scan estimators}, or $k$-NN graph scan estimators. These estimators can be used for doing nonparametric statistics and unsupervised learning on the graph (network) $G_n$.

Our sole assumption on the graph structure and active and inactive communities can be formulated in a local form concerning only a negligibly small sub-community of inactive users.
\begin{assum}
  \label{ass:1}
  For our search of inactive users, assume that there is an inactive vertex $v \in G$ with a full $k(n)$-neighborhood of inactive nearest neighbors, and that
  \begin{equation}\label{69}
    \lim_{n \to\infty} \frac{\, k (n) \,}{\, \log n \,} = \infty \,.
  \end{equation}
\end{assum}

Right below we clarify the terminology used in the Assumption. Let $v \in V$ be a vertex of $G$ and let $m \in \mathbb{N}$ be any number. We define a \emph{nested (multilevel) neighborhoods} of $v$ in $G$ as follows. First define
$\mathfrak{N}_0 (v) := v$, and then define recursively $\mathfrak{N}_1 (v) \,:=\, \{\, v_1 \in V \,|\, (v, v_1) \in E \,\}$.
More generally, for any natural number $i$,
\begin{equation*}\label{42}
  \begin{split}
    \mathfrak{N}_i (v) := \{\, v_i \in V \,|\,\ &\text{there exists} \ v_{i-1} %MJ: was e_..
    \in \mathfrak{N}_{i-1} (v)\\
    &\text{such that}\quad (v_{i-1}, v_i) \in E\}.
  \end{split}
\end{equation*}
%\end{defn}
%\begin{defn}\label{Definition_Neighborhood_Order}
For natural $i$, \emph{descendance level sets of $i$-th order} for $v$ in $G$, are defined for $i \geq 1$ as
\begin{equation}\label{43}
\mathfrak{D}_i (v) \,:=\, \mathfrak{N}_i (v) \, \setminus \mathfrak{N}_{i-1} (v) \,.
\end{equation}
For $i = 0$, we have set for convenience
\begin{equation}\label{66}
\mathfrak{D}_i (v) \,:=\, v \,.
\end{equation}
For any $k \in \mathbb{N}$, a \emph{(full) neighborhood of $k$-th order} of $v$ in $G$ is defined as the union
\begin{equation}\label{44}
    \Omega_k (v) \,:=\, \bigcup\nolimits_{0 \leq i \leq k} \mathfrak{N}_i (v) \,.
\end{equation}

%\end{defn}
Lemma~\ref{Lemma_7} states a simple relation between $\mathfrak{D}$ and $\Omega$.
\begin{lemma}\label{Lemma_7}
Let $\mathfrak{D}_i$ and $\Omega_k$ be defined as above. Then,
\begin{equation}\label{67}
\mathfrak{D}_i (v) \,=\, \mathfrak{N}_i (v) \, \setminus \bigcup\nolimits_{m = 0}^{i-1} \mathfrak{N}_i (v),
\end{equation}
and
\begin{equation}\label{45}
    \Omega_k (v) \,:=\, \bigcup\nolimits_{0 \leq i \leq k} \mathfrak{D}_i (v) \,.
\end{equation}
\end{lemma}

%\begin{defn}\label{Definition_Full_Neighborhood_Order}
For any $m \in \mathbb{N}$, a \emph{full $m$-neighborhood} $\mathfrak{B}^{+}_m (v)$ of $v$ in $G$ is $\Omega_r (v)$ such that $|\, \Omega_{r-1} (v) \,| \,<\, m$, but $|\, \Omega_r (v) \,| \,\geq\, m$.

For any $m \in \mathbb{N}$, an \emph{exact $m$-neighborhood} $\mathfrak{B}_m (v)$ of $v$ in $G$ is any subset $\mathfrak{B}_m (v) \,\subseteq\, \Omega_r (v)$ such that $\Omega_{r-1} (v) \,\subseteq\, \mathfrak{B}_m (v)$ and $|\, \mathfrak{B}_m (v) \,| \,=\, m$.
%\end{defn}

\noindent In a graph $G = (V, E)$, for a set of vertices $K \subseteq V$, the total sum of values observed over $K$ will be denoted as

\begin{equation}\label{6}
S_K \,=\, \sum_{v \,\in\, K} X_v \,.
\end{equation}

\begin{definition}[$k$-NN scan estimator]\label{Definition_KNN_Scan_Estimator}

Let $\mathcal{K}_0$ be any collection of exact $k(n)$-neighborhoods of all vertices in $G$:

\begin{equation}\label{70}
\mathcal{K}_0 \,:=\, \{\, \mathfrak{B}_{k(n)} (v) \,|\, v \in V \,\} \,.
\end{equation}

Set
\begin{equation}\label{71}
\widehat{K} \,:=\, \arg\min_{K \in \mathcal{K}_0} \{ S_{K} \}
\end{equation}
and define a \emph{(sublevel) $k$-NN scan estimator} as
\begin{equation}\label{72}
\widehat{a} \,:=\, \frac1{k(n)} \sum_{v \in \widehat{K}} X_v \,.
\end{equation}
\end{definition}

The $k$-NN scan estimator for inactive vertices can be computed via the following Algorithm 1.

\begin{algorithm}[t]
   \caption{$k$-NN scan estimator}
   \label{alg:kNNscan}
\begin{algorithmic}[1]
   \STATE \textsf{\small Phase 1 (decentralized).} Pick one arbitrary $k(n)$-neighborhood per vertex in $V$, and compute the average of $X_v$ over this neighborhood.
   \STATE \textsf{\small Phase 2 (collaborative).} Identify a node with smallest average, over its neighborhood $\widehat{K}$.
   \STATE Output $\widehat{a} \,:=\, \frac1{k(n)}\sum_{v \in \widehat{K}} X_v$.
\end{algorithmic}
\end{algorithm}

%\noindent \textbf{Algorithm (k-NN scan estimator (sublevel)).}
%%
%\begin{itemize}
%\item Phase 1. Pick one arbitrary $k(n)$-neighborhood per vertex in $V$, and compute the average of $X_v$ over this neighborhood.
%
%\item Phase 2. Select a neighborhood $\widehat{K}$ with the smallest average.
%%In case if several sums are equal, take an arbitrary one.
%%
%%\item Step 3.
%Output $\widehat{a} \,:=\, \frac1{k(n)}\sum_{v \in \widehat{K}} X_v$.
%
%%\item Step 4. Define $\widehat{F}$ as the empirical d.f. on $\widehat{\mathcal{K}}$.
%\end{itemize}

\begin{example}
  An important special case of graphs is given by lattices. Suppose we have a noisy two-dimensional pixelized image. We are interested in detection of objects that have an unknown color. This color has to be different from the colour of the background. It is also assumed that on each pixel we have random noise that has an unknown nonparametric distribution. This type of model is typical for cryo-electron microscopy. Typically, each cryo-EM picture contains a large number of particles. Particles have unknown, irregular, nonconvex and different shapes and sizes. The only common property for all cryo-EM pictures is that particles are darker than the background and that the noise has a completely unknown irregular distribution. However, both the background intensity and the particle intensity vary from image to image and are not known in advance.

  It is possible to view digital images as networks, where individual pixels correspond to vertices. Since the initial noisy image can be naturally viewed as a square lattice graph, where $k^2$-nearest neighbors correspond to a $k \times k$ subsquare on the screen, we see that a popular sliding window estimator is a special case of the $k$-NN scan estimator.

  This paradigm was used to solve a number of nonparametric unsupervised learning problems in image analysis and cryo-EM applications (see \citep{langovoy_habeck_schoelkopf}, \citep{langovoy_habeck_schoelkopf_JSM}). Many results on discrete spatial scan estimators from \citep{langovoy_habeck_schoelkopf} and \citep{langovoy_habeck_schoelkopf_JSM} are special cases of results of the present paper.

%In this special case, we call the $k$-NN scan estimator a discrete spatial scan estimator \citep{langovoy_habeck_schoelkopf}, \citep{langovoy_habeck_schoelkopf_JSM}.
\end{example}

%It is possible to define scan estimators that use either \emph{full} $k(n)$-neighborhoods of all vertices, or \emph{all} exact $k(n)$-neighborhoods. These estimators are consistent under rather general assumptions as well, even though the corresponding limiting distributions differ from those of the scan estimator of this paper. However, the active adversary assumption needs to be altered to ensure consistency for each of these estimators.

In view of the symmetry of the results between sub- and super-level cases, in this paper only the sublevel case is considered in details. For completeness, the superlevel scan estimator and the estimator of the crawler's performance are described in the Appendix.

%below.

\section{Consistency of k-NN scan estimators}
\label{sec:consistency}
In this section, we establish strong consistency of the proposed $k$-NN scan estimator under the assumption of bounded noise. % , and to derive its distribution or rates of convergence, we might impose some additional assumptions on our model. We have chosen to consider the case of bounded noise in this paper. Assuming that the noise is bounded, we seek to
Suppose, therefore that there is a constant $M > 0$ such that for all $v \in G_n$
\begin{equation}\label{18}
    |\, \varepsilon_v \,| \,\leq\, M \quad \text{almost surely}.
\end{equation}
The bounded noise case is not the only case when the $k$-NN scan estimator is strongly consistent. We mainly use condition (\ref{18}) to establish tight nonasymptotic performance guarantees for the estimator. Some form of asymptotic consistency can be established for unbounded noise from large nonparametric classes as well. This is illustrated in Section \ref{Section_Experiments}, where \emph{all} of our experiments are performed for an unbounded noise.

Let $K \subseteq V_n$ be any collection of vertices. Denote by $S_1(K, n)$ the \textit{number of active vertices} in $K$. % i.e.,
% \begin{equation}\label{Active_Subset_Cardinality}
%     S_1(K, n) \,=\,  | \{ v \in K \,| \, v \,\, \text{is active} \} |.
% \end{equation}
We prove the following statement that provides the foundation of model selection on graph neighborhoods.

\begin{proposition}\label{Proposition_7}
Let $K_0$ be any set of inactive vertices with $| K | \,=\, k(n)$, and let $\mathcal{K}$ be any collection of exact $k(n)$-neighborhoods. Define, for any $K \in \mathcal{K}$,
%\begin{align*}%\label{49}
%   &\mathfrak{R} (K, a, b, \{A_v\}) \,:=\,\\
%   &\qquad (b-a) \, S_1(K, n) \,+\! \sum_{\{ v \in K \,| \, v \,\, \textit{active} \}} ( A_v - b )  \,.
%\end{align*}

\begin{equation*}
   \mathfrak{R} (K, a, b, \{A_v\}) \,:=\, \qquad (b-a) \, S_1(K, n) \,+\! \sum_{\{ v \in K \,| \, v \,\, \textit{active} \}} ( A_v - b )  \,.
\end{equation*}

Then, we have the following bound:
%\begin{small}
%  \begin{align}\label{52}
%    & P\, (\widehat{K} \in \mathcal{K}, S_{\widehat{K}} < S_{K_0}) \leq  \\
%    &\sum_{K \,\in\, \mathcal{K}}  \exp \Bigr(\! - \frac{ 3 \, {\bigr[\mathfrak{R} (K, a, b, \{A_v\})]}^2 }{12 \, \sigma^2 \,|\,K \setminus K_0\,|\, +\, 4 \, M \,\cdot \, \mathfrak{R} (K, a, b, \{A_v\})} \Bigr) \,. \notag
%  \end{align}
%\end{small}

\begin{equation}\label{52}
    P\, (\widehat{K} \in \mathcal{K}, S_{\widehat{K}} < S_{K_0}) \, \leq \,  
    \sum_{K \,\in\, \mathcal{K}}  \exp \Bigr(\! - \frac{ 3 \, {\bigr[\mathfrak{R} (K, a, b, \{A_v\})]}^2 }{12 \, \sigma^2 \,|\,K \setminus K_0\,|\, +\, 4 \, M \,\cdot \, \mathfrak{R} (K, a, b, \{A_v\})} \Bigr) \,. 
\end{equation}

%\end{widetext}
\end{proposition}

Notice that the statement of Proposition \ref{Proposition_7} is actually \emph{non-asymptotic}: even though the bound depends on $n$, it is valid for small $n$ as well. Additionally, the bound (\ref{52}) depends on the subgraph $K$, and automatically tightens for those subgraphs that have more active vertices. This can be used to show that (\ref{52}) is essentially tight, with or without the adversary, also in finite-sample, non-asymptotic cases.

The bound in Proposition \ref{Proposition_7} is based on the first Bernstein inequality \citep{Bernstein_1937_Inequality}, which substantially improves on the Hoeffding inequality in case if $\Var \, \varepsilon_v \,\ll\, M^2$, which is the case for many distributions of interest. Additionally, this makes the bound (\ref{52}) extensible to weakly-dependent random variables as well, unlike the corresponding bound that relies on the Hoeffding inequality.

% From now on, let us denote
% \begin{equation}\label{53}
% \widehat{a}_K \,:=\, \overline{X}_{\widehat{K}} \,.
% \end{equation}

The following is the master theorem governing consistency of graph scan estimators. We understand consistency in its \textit{strong classical form}: as the sample size increases, the algorithm learns the true value of the parameter correctly with probability approaching 1.

\begin{theorem}\label{Theorem_2}
Let $\mathcal{K}_{\mathcal{A}}(n)$ be a collection of exact $k(n)$-neighborhoods that a $k$-NN graph scan algorithm $\mathcal{A}$ uses for estimation. Suppose $K_0 \in \mathcal{K}_{\mathcal{A}}(n)$ and
\begin{equation}
\lim_{n \rightarrow \infty} \, \frac{\, k (n) \,}{\, \log |\mathcal{K}_{\mathcal{A}}(n)| \,} \,=\, \infty \,,
\end{equation}

\begin{equation}
\lim_{n \rightarrow \infty} \, |\mathcal{K}_{\mathcal{A}}(n)| \, \exp (-n) \,=\, 0 \,.
\end{equation}
Then the scanning algorithm $\mathcal{A}$ leads to the graph scan estimator $\widehat{a} \,=\, \overline{X}_{\widehat{K}}$ that is a consistent estimator of $a$, regardless of the adversary's strategy.
\end{theorem}
As a corollary of this general statement, we derive consistency of the estimator of the present paper. Proofs of all the results can be found in Appendix.
\begin{theorem}\label{Theorem_1}
Suppose that $k(n)$ satisfies Assumption~\ref{ass:1}. Then,
% \begin{equation}\label{56}
% \lim_{n \rightarrow \infty} \, \frac{\, k (n) \,}{\, \log n \,} \,=\, \infty \,.
% \end{equation}
the graph scan estimator $\widehat{a} \,=\, \overline{X}_{\widehat{K}}$ is a consistent estimator of $a$, regardless of the adversary's strategy.
\end{theorem}

%Experimental results for real and synthetic graphs, as well as for the cases with and without adversaries can be found in Section \ref{Section_Experiments}. Scalability of the method is established and the algorithm's distributed realization is discussed in Section \ref{Section_Optimization}.

\section{Scalable algorithms and computation}\label{Section_Optimization}
In view of the huge size of many networks of current interest, and also due to the fact that we typically do not get the chance to observe the network in its entirety but rather can only query it in small parts, we need fast and scalable algorithms to implement our estimators. Moreover, these algorithms should be provably statistically reliable and consistent. In this section we discuss some scalability aspects of our estimators; however, to avoid detracting from our primary modeling focus, we defer a thorough empirical evaluation to the future.

\paragraph{Distributed computation.}
Note that the computation of our proposed estimator shown in Algorithm \ref{alg:kNNscan} is fully decentralized, breaking up into computation of one neighborhood per graph node (Phase 1). However, there is a necessary centralized communication phase at the end that collects all averages (a single number) per neighborhood and only keeps the smallest one out of those (Phase 2). This second phase is easily implemented as a typical map-reduce operation over the nodes, the `reduce' operation being the $\min$ operator.

\begin{proposition}
  Let $k(n)$ be the neighborhood size. One can compute a $k$-NN scan estimator using $\mathcal{O}(k(n))$ operations per node in the decentralized phase (Phase 1) of Algorithm~\ref{alg:kNNscan}, followed by a single communication round to compute the minimum over the graph.
\end{proposition}

\paragraph{Parallelization.}
The main computational cost is encountered in Phase 1 of Algorithm~\ref{alg:kNNscan}. Fortunately, the computations in this phase decouple completely, enabling our algorithm to be directly applicable to huge networks. In this setting, we assume the nodes of the graph are partitioned over a set of compute agents, each responsible only for the estimators over its own nodes. In the extreme case, in an IoT setting as noted below, every graph node has its own compute agent responsible for it. All independently execute the same BFS style algorithm to build an average over the $k(n)$ sized neighborhood per node (naively needing the communication of $k(n)$ numbers back to node $v$).

%\paragraph{IoT networks.}
In an IoT framework, we naturally have a processor built-in at each graph node. We therefore readily use those processors (that have immediate access to local observations, and communicate in small neighborhoods) for parallel estimation using Algorithm~\ref{alg:kNNscan}.

\section{Experimental results}\label{Section_Experiments}

\subsection{Simulated large networks, no adversary}

To illustrate performance of our method on a larger scale, we constructed an artificial network that consists of two subgraphs - one of which is made of a million nodes where each of these nodes is connected to 3 nodes chosen uniformly at random within this subgraph. The other subgraph consists of a thousand nodes and, in a similar way, each of these nodes is connected to 3 nodes chosen uniformly at random within this second subgraph. All the 1000 nodes in the smaller subgraph are labeled as inactive, while nodes in the bigger subgraph are labeled as active or inactive randomly with equal probability. This way, we guaranteed that there is an inactive community of size $k(n) \,=\, 1000$. Furthermore, we pick 20 random pairs of nodes - one from the bigger subgraph and the other from the smaller one and connect the two. Number of nodes in the graph is 1,001,000 and number of edges is up to 3,003,002.

Activity levels in the network are set to $a=2$ for the inactive nodes and $A_v=10$ for the active ones. These weights are corrupted by adding a Gaussian noise of mean $0$ and variance $\sigma^2 = 1$ before feeding them to the algorithm.

The estimates returned for the above example using our proposed approach are in accordance with the theory. The algorithm underestimates $a$ when we pick the value of the parameter $k$ that is too small (this value is corresponding to the size of the neighborhood). The estimate's value increases with an increased value of $k$. Halfway between the two extreme values we obtain very accurate estimates. Naturally, these values of $k$ that return the optimal estimate ($k = 500$ or $k=1000$ in this case) can also be considered suggestive of the inactive community sizes inside the network.

If we change the size of the smaller group to 10,000, we observe that the $k = 1000$ setting of the algorithm returns an estimate of $\widehat{a}=1.75$ on average over 400 experiments, implying that we have to scan with bigger neighborhoods to achieve consistent inference. Indeed, in this case $k(n)=10,000$.

The histogram in Figure 1 illustrates the distribution of the $k$-NN graph scan estimators for different values of $k$. In particular, for $k=500$ the empirical mean of the estimator is 1.91914586 with variance 0.0341272087, while for $k=1000$ we have estimated mean 2.03303998 with variance 0.031305042. The true value $a=2$ is, indeed, very close.

\begin{figure}\label{Politics_Graph_Histogram}
\centering
	\includegraphics[scale=0.06]{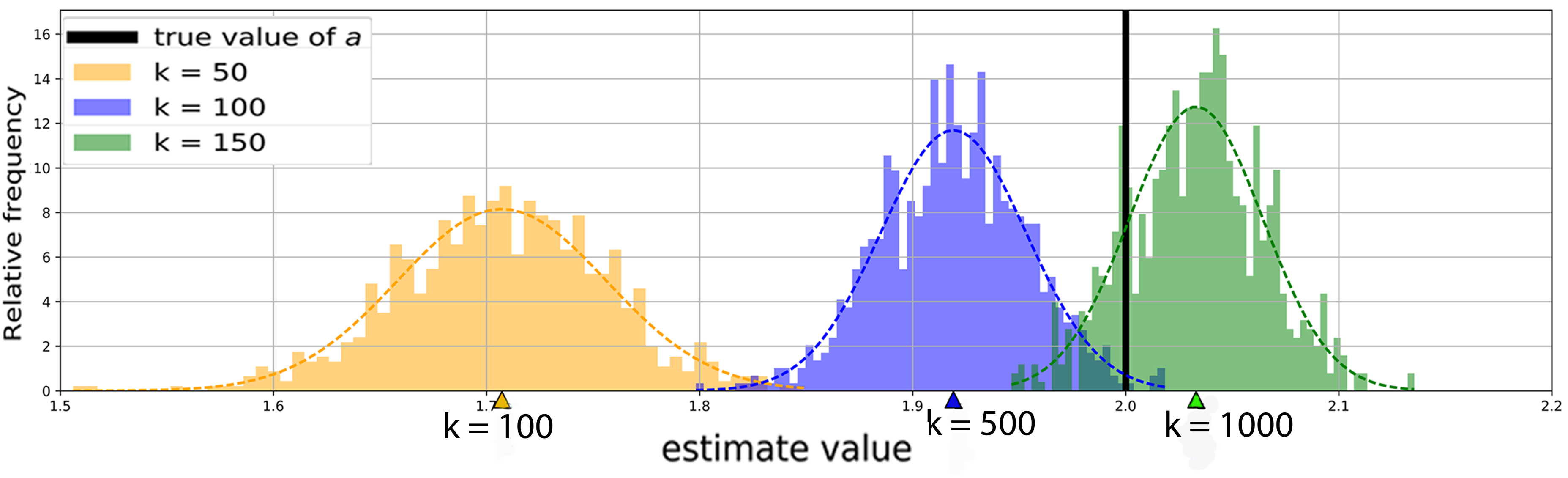}	\vspace{-3mm}
\caption{
Graph scan estimators for the artificial large graph.
}
\end{figure}

\subsection{Real network structure, no adversary}

In this example, we use the real community structure of political blogs from the 2004 U.S. Presidential Election. Paper \citep{adamic2005political} studied the linking patterns of political blogs. These blogs can be naturally classified into two classes, the liberal ones and the conservative ones. This is suitable for our framework. The number of libertarian, independent, or moderate blogs was negligible at the time.

In \citep{adamic2005political}, a description of the network of over 1000 blogs is presented, based on a single day snapshot that included blogrolls. Two blogs share an edge if one of the two cites the other one. The blogs were categorized manually. It turned out that neither directory labels relying on self-reported or automated categorizations, nor the manual labels were 100\% accurate, with an error probability that has an unknown distribution.

There were 1494 blogs in total, with 759 liberal and 735 conservative. The structure of the underlying linking graph was rather complex, as 91\% of the links originating within either the conservative or liberal communities stay within that community, but the number of intercommunity links was non-negligible as well. Figure 2 illustrates the structure of the political blogging graph.

\begin{figure}\label{Politics_Graph}
\centering
	\includegraphics[scale=0.26]{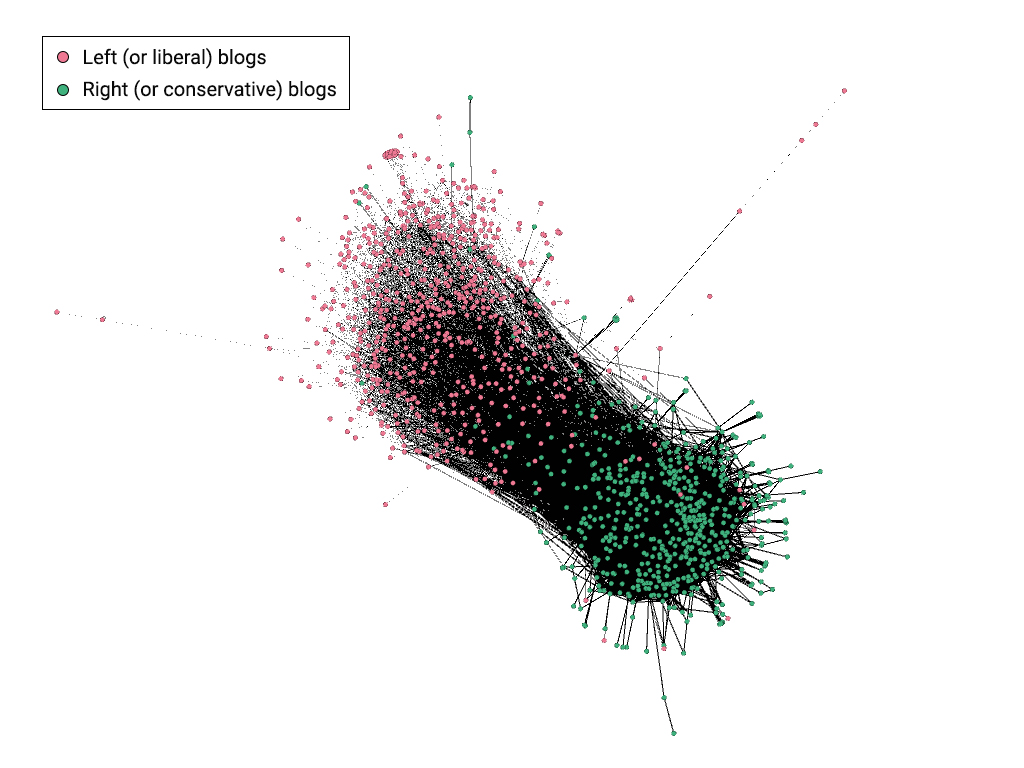}	\vspace{-3mm}
\caption{
Political blogosphere graph for the 2004 Elections.
}
\end{figure}

\begin{figure*}[t]\label{Politics_Graph_Histogram}
\centering
	\includegraphics[scale=0.35]{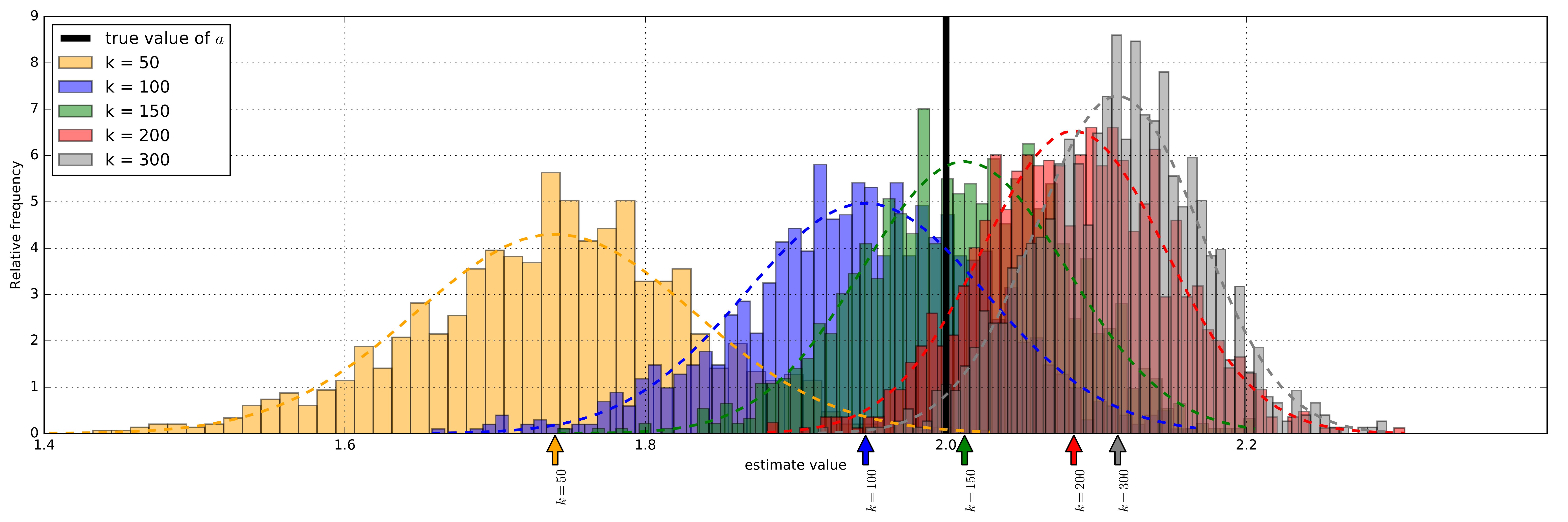}\vspace{-3mm}
\caption{
Graph scan estimators for the 2004 Elections graph.
}
\end{figure*}

To create a challenge to our method, we add an additional complication to this dataset by corrupting the values attributed to blogs by Gaussian white noise. We assigned the same activity levels as in Section 6.1. Notice that this type of noise is unbounded and so does not satisfy conditions of the consistency theorem. The graph size is also relatively small, so we are far from the asymptotic regime. However, the graph scan estimator produced surprisingly accurate results.

The histogram in Figure 3 illustrates the distribution of the $k$-NN graph scan estimators for different values of $k$. It is apparent that in this graph there is an inactive neighborhood of size close to 150, but not more. An estimator for $k=150$ is surprisingly accurate.

%\subsection*{Performance against an adversary}

\subsection{Simulated large network, two types of adversaries}

In order to study the impact of an adversary on the performance of the proposed algorithm, the following experiment is designed.

Let us consider the following two scenarios. Scenario 1: the adversary decides to act locally and chooses to influence those active nodes that are present in the best neighborhood $\hat{K}$ of our choice, thus hoping to throw off our specific Algorithm.

Scenario 2: the adversary is strong and decides to use brute force approach, influencing all the active nodes in the network.

For our experiment, we assumed that when the adversary 'influences' a node, the activity level for this node is changed to a large number $10^{6}$. We then run our algorithm on the new set of altered activity levels and generate a new estimate $\hat{a}^{(2)}$ corresponding to a new 'best neighborhood' $\hat{K}^{(2)}$.

Over $100$ such experiments for the value of $k = 500$, there happen to be only $23$ cases where $\hat{K}$ obtained after first running our algorithm contains at least one active node. Hence for the remaining $77$ cases, the estimate $\hat{a}^{(2)}$ obtained after running the algorithm over the altered values by either adversary - Scenario 1 or Scenario 2, everything remains the same for the Algorithm since the new best neighborhood $\hat{K}^{(2)}$ would remain the same as the one obtained previously ($\hat{K}$). For the $23$ times where we have an active node in $\hat{K}$, the average of the estimate $\hat{a}$ when the algorithm is first run has a mean of $1.9320$ and a standard deviation of $0.0339$. When we run the algorithm after the action of the weak adversary, the estimate $\hat{a}^{(2)}$ has a mean of $1.9399$ and a standard deviation of $0.0333$. When the adversary is a strong one, our estimate $\hat{a}^{(2)}$ has a mean of $1.9423$ and a standard deviation of $0.0322$.

However for the value of $k = 1000$, we reported $97$ cases where the neighborhood $\hat{K}$ contains one or more active nodes and for the remaining $3$ we have no active nodes within $\hat{K}$. The mean of the estimates $\hat{a}$ when the algorithm runs for the first time was $2.0293$ with a standard deviation of $0.0318$, while after the influence of the weak adversary the estimates $\hat{a}^{(2)}$ had a mean of $2.0472$ and a standard deviation of $0.0438$ and in presence of the strong adversary, the estimates had a mean of  $1.9220$ and a standard deviation of $0.0354$.

This shows that our method exhibits remarkable stability against both crafty and brute force adversaries, even under conditions that are more general than the ones in our consistency theorem.

%In the Appendix, we describe our experiments with adversaries that can, in addition, act multiple times instead of just one. These experiments show that the estimator can still be very robust in these settings.

\subsection{Experiments with simulated large networks and multi-step adversaries}

%For a weak adversary that can only influence the active nodes in our chosen neighborhood,

In this section, we show that our method can be robust even in those cases where the adversary can act not once, but as many times as he chooses. We build a multi-step multi-stage experiment, where at each step, firstly, we generate an estimate by choosing an optimal neighborhood with our scanning algorithm and, secondly, the adversary acts to spoil our inference by influencing active nodes in the neighborhood that we just selected. This game can be repeated as many times as the adversary chooses. Notice that this is a much more general setup than we permitted in Section 2.1.

For a set of values for $k(n)$ (the size of the scanning neighborhood), we perform 100 experiments for each case. We let the adversary play as long as it made sense, and it happens that we only needed up to 10 steps. In all these cases, we always won the game, by reaching a stage where we select a neighborhood that doesn't contain any active node; therefore, after this step the algorithm is immune to adversary's influence.

The table \ref{table:1} shows the estimates returned in these experiments, and the number of steps $N_w$ required before winning against the adversary. When $N_w = 0$, the neighborhood chosen by our algorithm in the very first attempt (before adversary's influence) does not contain any active nodes and thus we win at the $0$-th step. It can be seen that as $k$ increases, the value of the estimate improves, getting close to the true value of $a\,=\,2$ for the optimal range of values $k=500$ to $k=800$, but then we often need more steps to beat the adversary. Thus choosing the parameter $k$ can be considered as facing a \emph{trade-off} between the accuracy of the estimate and the robustness against an adversary.

\begin{table}
\centering
\begin{tabular}{ |p{1.8cm}||p{0.75cm}|p{0.75cm}|p{0.75cm}||p{0.75cm}||p{0.75cm}|  }

 \hline
$k$ & 100 & 200 & 500 & 700 & 800 \\
\hline
\hline
 $\#(N_w=0)$   & 98    & 98    &78 & 56 & 55\\
 $\#(N_w=1)$   & 2     & 2     &16 & 22 & 17\\
 $\#(N_w=2)$   & 0     &  0    & 3 & 16 & 12\\
 $\#(N_w = 3)$ & 0     &  0    & 3 & 5  & 9\\
 $\#(N_w \geq 4)$   & 0     &  0    &  0      & 1     & 7\\
 \hline
 $\mu_{\hat{a}}$    & 1.705 & 1.815 &  1.913  & 1.949 & 1.961\\
 $\max{\hat{a}}$    & 1.806 & 1.903 &  1.980  & 2.030 & 2.031\\
 $\min{\hat{a}}$    & 1.603 & 1.686 &  1.824  & 1.848 & 1.879\\

 \hline
 $\sigma_{\hat{a}}$ & 0.048 & 0.040 & 0.033   & 0.034 & 0.029\\
%  \#(mean of \hat{a})   &DZ & DZA&  012\\

\hline
\end{tabular}
\caption{Multi-step adversary}
\label{table:1}
\end{table}

\smallskip
\noindent {\bf Acknowledgments.} The authors would like to thank Suvrit Sra for helpful discussions. \\

\bibliographystyle{plainnat}
%\bibliographystyle{cj}
%\bibliography{Randomized_Algorithms_and_Percolation}
\bibliography{Networks_IoT_Bibliography}

%\noindent {\bf Appendix.}\\

\section*{Appendix}

\section{Graph scan estimators for the crawler and for active vertices}

It is possible to define scan estimators (for each of the three types discussed in this paper) that use either \emph{full} $k(n)$-neighborhoods of all vertices, or \emph{all} exact $k(n)$-neighborhoods. These estimators are consistent under rather general assumptions as well, even though the corresponding limiting distributions differ from those of the scan estimator of this paper. However, the active adversary assumption needs to be altered to ensure consistency for each of these estimators.

\subsection{Superlevel graph scans for active vertices}

For inference on active clusters we would have to assume that there exists at least one active cluster that contains a full $\varphi (n)$-neighborhood of vertices, and that
  \begin{equation}\label{68}
    \lim_{n \to\infty} \frac{\, \varphi (n) \,}{\, \log n \,} = \infty \,.
  \end{equation}

Construction of superlevel $k$-NN scan estimators for the active level ($b$ in our notation) requires the following dual assumption (\emph{cf.} \eqref{3}):
\begin{equation}\label{65}
\text{for each}\ v \in V,\ \left\{
           \begin{array}{ll}
           A_v = b > a, & \hbox{if $v$ is active;} \\
           A_v \leq a, & \hbox{if $v$ is normal.}
           \end{array}
         \right.
\end{equation}

Now the above graph scan algorithm can be inverted to get a $k$-NN scan estimator for active vertices instead. The only modification would be that this time
\begin{equation}
\widehat{K} \,:=\, \arg\max_{K(v) \subseteq V: \, |\,K(v)\,| \,=\, \varphi(n), \, v \in V} \{ S_{K(v)} \}
\end{equation}
and
\begin{equation}
\widehat{b} \,:=\, {\varphi(n)}^{-1}\sum\nolimits_{v \in \widehat{K}} Y_v \,.
\end{equation}

In view of the symmetry of the results between the sublevel and the superlevel cases, in this paper only the sublevel case is considered in details.

\subsection{Graph scans for learning the crawler's performance}

Our framework involves a third-party network crawling algorithm that provides us the data about the users of the network. We propose here a modified graph scan estimator to estimate the unknown distribution $F$ of the misclassification rate of the crawling algorithm. The difficulty here is again that we do not know which vertices are active or inactive. However, this problem is solved by a combination of the graph scan estimator with the empirical distribution function estimator.

%\begin{defn}\label{Definition_Spatial_Scan_Estimator_Distribution}

Let $\widehat{K}$ be defined by (\ref{71}) of Definition \ref{Definition_KNN_Scan_Estimator}. A \emph{graph scan estimator for the crawler's error} distribution $F$ would be
\begin{equation}\label{Scan_Estimator_Noise_Distribution}
    \widehat{F} (t) \,:=\, \frac{\, 1 \,}{\, |\widehat{K}| \,} \sum_{v \in \widehat{K}} \mathbbm{1}_{ \{ X_v \leq t \} }  \,.
\end{equation}

The following variation gives an \emph{unbiased} consistent estimator for the noise variance $\sigma^2$:
\begin{equation}\label{Scan_Estimator_Noise_Variance}
    \widehat{\sigma}^2 \,:=\, \frac{\, 1 \,}{\, |\widehat{K}| - 1 \,} \sum_{v \in \widehat{K}} {(\, X_v - \widehat{a} \,)}^2 \,.
\end{equation}
Both scan estimators $\widehat{\sigma}^2$ and $\widehat{F} (t)$ can be easily calculated once $\widehat{a}$ is calculated. We conjecture that these are consistent estimators of $\sigma^2$ and $F$, respectively. Moreover, there is evidence that $\widehat{F}$ is uniformly $\sqrt{\, k(n) \,}$-consistent in $\| \cdot \|_{\infty}$-norm.

\section{Proofs}

In this section, we start with deriving probabilistic bounds for, and statistical properties of, sublevel graph scan estimators for the case of general unbounded noise of finite variance. Corresponding results for superlevel scan estimators are analogous. We will not study consistency of scan estimators in this general case.

Notice that in the model formulated in Section 2, potentially, $A_v \,=\, A_v ( \{ X_v \,|\, v \in V_n \}, G_n, \mathcal{A}(G_n) )$ for all active $v$, where $\mathcal{A}$ is the algorithm we use for inference, and $\mathcal{A}(G_n)$ is the collection of all the steps we would perform together with values of all the random variables that we will generate. Unlike the case of inactive vertices, this is not a collection of numbers, but rather a collection of random variables with a complex and unknown mutual dependence structure. Therefore, $\{ X_v \,|\, v \in V_n \}$ forms a collection of random variables too. We do not require this collection to be neither completely, nor pairwise independent, nor having identical marginal distributions. In fact, we allow for strong and long-distance correlations (within bounds given by (\ref{3})) between observations on active vertices.

%Therefore, potentially, $A_v \,=\, A_v ( \{ X_v \,|\, v \in V_n \}, G_n, \mathcal{A}(G_n) )$ for all active $v$, where $\mathcal{A}$ is the algorithm we use for inference, and $\mathcal{A}(G_n)$ is the collection of all the steps we would perform together with values of all the random variables that we will generate. Unlike the case of inactive vertices, this is not a collection of numbers, but rather a collection of random variables with a complex and unknown mutual dependence structure. Therefore, $\{ X_v \,|\, v \in V_n \}$ forms a collection of random variables too. We do not require this collection to be neither completely, nor pairwise independent, nor having identical marginal distributions. In fact, we allow for strong and long-distance correlations (within bounds given by (\ref{3})) between observations on active vertices.

We will use a pairing trick for random variables in the proofs below. The trick is based on the following basic lemma.

\begin{lemma}\label{Lemma_2}
Let $\{ \varepsilon_{v}\}_{i=1}^{i=2r}$, for a natural number $r \geq 1$ be a collection of independent identically distributed random variables satisfying (\ref{2}). Define for $1 \leq i \leq r$

\begin{equation}\label{Delta_Definition}
  \Delta_i \,=\, \varepsilon_i \,-\, \varepsilon_{i+r} \,.
\end{equation}

The collection of random variables $\Delta_1, \Delta_2, \ldots , \Delta_r$ is completely independent. For all $i$, it holds that $\mathbb{E}\, \Delta_i \,=\, 0$, $Var\, \Delta_i = 2\,\sigma^2$, and $\Delta_i$ is symmetric in distribution around 0. Moreover, $\Delta_i \, \sim_{d} \, F(\,\cdot \,) * (1 - F(\, - \cdot\,))$.
\end{lemma}

\begin{proof}(of Lemma \ref{Lemma_2}):
The first four assertions are trivial. The last one follows since $\varepsilon_i \, \sim_{d} \, F(\,\cdot \,)$ and $-\,\varepsilon_i \, \sim_{d} \, 1 - F(\, - \cdot\,)$, where $\sim_{d}$ denotes equality in distribution.
\end{proof}

The following simple geometric lemma would be useful.

\begin{lemma}\label{Lemma_1}
Let $K_0 \,\subseteq\, V_n$ be any collection of inactive vertices, with $K_0 \,=\, k(n)$, and let $K \,\subseteq\, V_n$ be a collection of arbitrary  vertices, with $K \,=\, k(n)$. Denote by $S_1(K, n)$ a number of active vertices in $K$, i.e.

\begin{equation}\label{Active_Subset_Cardinality}
    S_1(K, n) \,=\,  | \{ v \in K \,| \, v \,\, \textit{active} \} |  \,.
\end{equation}

Then $|\,K \setminus K_0\,| \,=\, |\,K_0 \setminus K\,|\,\geq\, S_1(K, n) \,.$
\end{lemma}

\begin{proof}(of Lemma \ref{Lemma_1}):
The equality is obvious since $|\,K\,| \,=\, |\,K_0\,|$. The inequality follows since all the $S_1(K, n)$ of active vertices of $K$ must be contained in $K \setminus K_0$ anyways.
\end{proof}

The following proposition provides the main proxy for application of concentration of measure inequalities to graph scan estimators and to model selection on subgraphs. In essence, the proof is based on the pairing trick on the algebra of random variables.

\begin{proposition}\label{Proposition_5}
Assume $\{ X_v \,|\, v \in V_n \}$ satisfies the additive noise model given by (\ref{1}) and (\ref{3}), and let $\{ \varepsilon_v \}$ be an arbitrary collection of independent identically distributed random variables with mean 0. Let $K_0 \,\subseteq\, V_n$ be any collection of inactive vertices, with $K_0 \,=\, k(n)$, and let $K \,\subseteq\, V_n$ be a collection of arbitrary  vertices, with $K \,=\, k(n)$. Denote by $S_1(K, n)$ a number of active vertices in $K$. Then, for any $\alpha \,=\, \alpha (n) \in \mathbb{R}$

\begin{flalign}\label{Basic_Probability_Equation}
\noindent    & P\, (\, S_{\widehat{K}} < S_{K_0} - \alpha (n) \,) \,=\,  \\
\noindent    &P \, \Bigr(\, \frac{\, 1 \,}{\, k(n) \,}  \Bigr(\, S_1(K, n) (b - a) \,+\, \sum_{\{ v \in K \,|\, v \,\,\textit{active} \} } (A_v - b) \,+\, \alpha (n) \,\Bigr)  \,<\, \frac{\, 1 \,}{\, k(n) \,}  \,\sum_{i=1}^{|\,K \setminus K_0\,|} \Delta_i \, \Bigr) \,, \notag
\end{flalign}

\noindent where $\{ \Delta_i \}$ is an arbitrary collection of independent identically distributed random variables
with $\Delta_i \, \sim_{d} \, F(\,\cdot \,) * (1 - F(\, - \cdot\,))$.

\end{proposition}

\begin{proof}
%As before, let $\varepsilon'_{ij} \,=\, \varepsilon_{ij} - a$ be the mean-centered noise. In order for some square $K$ to be chosen as $\widehat{K}$, it has to beat $K_0$ first. In other words, necessarily $\widehat{a} \,=\, \overline{a}_{\widehat{K}} \,\leq\, \overline{a}_{K_0}$. But by (\ref{9})

\begin{equation}\label{12new}
  P\, (\, S_{\widehat{K}} \,<\, S_{K_0} - \alpha (n) \,) \, =  \\
\end{equation}

\begin{eqnarray*}
   &=& P \, \Bigr(\, \sum_{v \,\in\, K \setminus \{ v \in K \,|\, v \,\,\textit{active} \}  } X_v \,+\, \sum_{\{ v \in K \,|\, v \,\,\textit{active} \}  } X_v  \,<\, \sum_{v \,\in\, K_0 } X_v \,  - \alpha (n) \,\Bigr) \,  \\
  &=& P \, \Bigr(\, \sum_{v \,\in\, K \setminus \{ v \in K \,|\, v \,\,\textit{active} \}  } (a + \varepsilon_v) \,+\, \sum_{\{ v \in K \,|\, v \,\,\textit{active} \}  } (A_v + \varepsilon_v)  \,<\, \sum_{v \,\in\, K_0 } (a + \varepsilon_v) \,  - \alpha (n) \,\Bigr)  \\
  &=& P \, \Bigr(\, \sum_{v \,\in\, K \setminus \{ v \in K \,|\, v \,\,\textit{active} \}  } (a + \varepsilon_v) \,+\, \sum_{\{ v \in K \,|\, v \,\,\textit{active} \}  } ((A_v - b) + (b - a) + (a + \varepsilon_v))  \,<\, \sum_{v \,\in\, K_0 } (a + \varepsilon_v) \,  - \alpha (n) \,\Bigr)  \\
   &=& P \, \Bigr(\, \sum_{v \,\in\, K } (a + \varepsilon_v) \,+\,  (b - a) \, S_1(K, n) \,+\, \sum_{\{ v \in K \,|\, v \,\,\textit{active} \} } (A_v - b)   \,<\, \sum_{v \,\in\, K_0 } (a + \varepsilon_v)  - \alpha (n) \,\Bigr) \\
    &=& P \, \Bigr(\, \sum_{v \,\in\, K } \varepsilon_v \,+\,  (b - a) \, S_1(K, n) \,+\, \sum_{\{ v \in K \,|\, v \,\,\textit{active} \} } (A_v - b) \,<\, \sum_{v \,\in\, K_0 } \varepsilon_v - \alpha (n) \,\Bigr) \\
    &=& P \, \Bigr(\,  (b - a) \, S_1(K, n) \,+\, \sum_{\{ v \in K \,|\, v \,\,\textit{active} \} } (A_v - b) \,+\, \alpha (n)  \,<\, \sum_{v \,\in\, K_0 } \varepsilon_v \, - \, \sum_{v \,\in\, K } \varepsilon_v \,\Bigr) \\
    &=& P \, \Bigr(\, (b - a) \, S_1(K, n) \,+\, \sum_{\{ v \in K \,|\, v \,\,\textit{active} \} } (A_v - b) \,+\, \alpha (n) \,<\, \sum_{v \,\in\, K_0 \bigcap K } 0 \, + \, \sum_{v \,\in\, K_0 \setminus K } \varepsilon_v \, - \, \sum_{v \,\in\, K \setminus K_0 } \varepsilon_v \,\Bigr) \,.
\end{eqnarray*}

%\noindent Notice that $\varepsilon'_{ij} \, \sim_{d} \, F(\,\cdot\,-a\,)$, which implies $\mathbb{E}\,\varepsilon'_{ij} \,=\,0$, $\Var \, \varepsilon'_{ij} \,=\,\sigma^2$, and also $\mathbb{E}\,(-\varepsilon'_{ij}) \,=\,0$, $\Var \, (-\varepsilon'_{ij}) \,=\,\sigma^2$. In particular, the law of large numbers and the central limit theorem are the same for $\varepsilon'_{ij}$ and $-\varepsilon'_{ij}$.
%
%We continue the proof of the theorem as follows.

\begin{eqnarray*}\label{12}
 &=& P \, \Bigr(\, \frac{\, S_1(K, n) (b - a) \,+\, \sum_{\{ v \in K \,|\, v \,\,\textit{active} \} } (A_v - b) \,+\, \alpha (n)\,}{\, k(n) \,}  \,<\, \frac{\, 1 \,}{\, k(n) \,}  \,\sum_{v \,\in\, K_0 \setminus K } \varepsilon_v \, - \, \frac{\, 1 \,}{\, k(n) \,}  \,\sum_{v \,\in\, K \setminus K_0 } \varepsilon_v \,\Bigr) \,.
\end{eqnarray*}

Now we perform the following pairing trick. Let us arbitrarily enumerate all pixels in $K \setminus K_0$ by numbers from 1 to $|\,K \setminus K_0\,|$. We also perform an arbitrary enumeration on $K_0 \setminus K$. Now in the equation (\ref{12}) we can regroup the difference of two sums into a single sum that consists of differences of pairs of random variables, and each pair is formed from the two noise variables with the same number. We would have $|\,K \setminus K_0\,|$ of those pairs. Call those random differences $\Delta_1, \Delta_2, \ldots , \Delta_{|\,K \setminus K_0\,|}$.

\noindent Applying Lemma \ref{Lemma_2} with $r \,=\, |\,K \setminus K_0\,|$, we continue to transform the above equations:

\begin{eqnarray}\label{13}
\nonumber  &=& P \, \Bigr(\, \frac{\, S_1(K, n) (b - a) \,+\, \sum_{\{ v \in K \,|\, v \,\,\textit{active} \} } (A_v - b) \,+\, \alpha (n)\,}{\, k(n) \,}  \,<\, \frac{\, 1 \,}{\, k(n) \,}  \,\sum_{v \,\in\, K_0 \setminus K } \varepsilon_v \, - \, \frac{\, 1 \,}{\, k(n) \,}  \,\sum_{v \,\in\, K \setminus K_0 } \varepsilon_v \,\Bigr) \\
    &=& P \, \Bigr(\, \frac{\, 1 \,}{\, k(n) \,}  \Bigr(\, S_1(K, n) (b - a) \,+\, \sum_{\{ v \in K \,|\, v \,\,\textit{active} \} } (A_v - b) \,+\, \alpha (n) \,\Bigr)  \,<\, \frac{\, 1 \,}{\, k(n) \,}  \,\sum_{i=1}^{|\,K \setminus K_0\,|} \Delta_i \, \Bigr) \,.
\end{eqnarray}

\end{proof}

In the sequel, we shall use the following general form of the Bernstein inequality (see \citep{Bernstein_1937_Inequality}).

\begin{proposition}\label{Bernstein_Inequality}
Let $X_1, \,\ldots\,, X_n$ be independent zero-mean random variables. Suppose that $|\, X_i \,| \,\leq\, M$ almost surely, for all $i$. Then for all $t \,>\, 0$

\begin{equation}\label{19}
    P\, \biggr(\, \sum_{i=1}^n \, X_i \,>\, t \,\biggr) \,\leq\, \exp \biggr( - \frac{t^2/2}{\sum_{i=1}^n \mathbb{E} X_i^2 + Mt/3 } \biggr) \,.
\end{equation}

\end{proposition}

We are ready to prove the basic cornerstone probabilistic inequality of this paper.

\begin{proposition}\label{Proposition_Basic_Inequality}
Let $K_0$ be any set of inactive vertices with $| K | \,=\, k(n)$, and let $K$ be any collection of $k(n)$ vertices. Define

\begin{equation*}%\label{49}
   \mathfrak{R} (K, a, b, \{Im\}) \,:=\, (b-a) \, S_1(K, n) \,+\! \sum_{\{ v \in K \,| \, v \,\, \textit{active} \}} ( A_v - b )  \,.
\end{equation*}

Then, if assumptions of Proposition \ref{Proposition_5} are satisfied and (\ref{18}) holds,

\begin{equation}\label{Basic_Inequality_Bernstein}
    P\, (\, \widehat{K} \,=\, K \,, S_{\widehat{K}} < S_{K_0} \,) \,\leq\,  \exp \biggr(\! - \frac{ 3 \, {\bigr[\mathfrak{R} (K, a, b, \{Im\})]}^2 }{12 \, \sigma^2 \,|\,K \setminus K_0\,|\, +\, 4 \, M \,\cdot \, \mathfrak{R} (K, a, b, \{Im\})} \biggr) \,. \notag
\end{equation}

\end{proposition}

\begin{proof}(of Proposition \ref{Proposition_Basic_Inequality}):
Notice that under the assumption (\ref{18}) we have $|\, \Delta_i \,| \,\leq\, 2\,M$ and $\mathbb{E}\,\Delta_i \,=\,0$. Setting $\alpha (n) \,=\,0$ for all $n$ and applying Proposition \ref{Proposition_5}, Lemma \ref{Lemma_2}, and the Bernstein inequality in the form given by Proposition \ref{Bernstein_Inequality}, we derive %for all $t \,>\, 0$:

%\begin{eqnarray}\label{20}
%% \nonumber to remove numbering (before each equation)
%  P\, (K \,\, \textit{is preferred to} \,\, K_0) &\leq&  P \, \Bigr(\,  (b - a) \, S_1^{(K)}(N) \,\leq\, \sum_{i=1}^{|\,K \setminus K_0\,|} \Delta_i \,\Bigr)
%\end{eqnarray}

\begin{align}\label{Basic_Probability_Equation}
  P\, (\, \widehat{K} \,=\, K \,, S_{\widehat{K}} < S_{K_0} \,) \,\leq\,    & P\, (\, S_{\widehat{K}} < S_{K_0} \,) \,=\,  \\
    &P \, \Bigr(\, \frac{\, 1 \,}{\, k(n) \,}  \Bigr(\, S_1(K, n) (b - a) \,+\, \sum_{\{ v \in K \,|\, v \,\,\textit{active} \} } (A_v - b) \Bigr)  \,<\, \frac{\, 1 \,}{\, k(n) \,}  \,\sum_{i=1}^{|\,K \setminus K_0\,|} \Delta_i \, \Bigr)  \notag
\end{align}

\begin{eqnarray}\label{21}
\nonumber   &=& \exp \, \biggr( -\, \frac{\, \Bigr(\, S_1(K, n) (b - a) \,+\, \sum_{\{ v \in K \,|\, v \,\,\textit{active} \} } (A_v - b) \Bigr)^2 \,}{\,2\,} \\
\nonumber & & \quad\quad\quad \,:\, \biggr(\sum_{i=1}^{|\,K \setminus K_0\,|} \mathbb{E}\, \Delta_i^2 \,+\, \frac{\,2 \, M \, \Bigr(\, S_1(K, n) (b - a) \,+\, \sum_{\{ v \in K \,|\, v \,\,\textit{active} \} } (A_v - b) \Bigr)\,}{\,3\,} \,\biggr)\biggr) \\
\nonumber    &=& \exp \, \biggr( -\, \frac{\,\Bigr(\, S_1(K, n) (b - a) \,+\, \sum_{\{ v \in K \,|\, v \,\,\textit{active} \} } (A_v - b) \Bigr)^2 \,}{\,2\,} \\
\nonumber & & \quad\quad\quad \,:\, \biggr(\sum_{i=1}^{|\,K \setminus K_0\,|}  Var\,\Delta_i \,+\, \frac{\,2 \, M \, \Bigr(\, S_1(K, n) (b - a) \,+\, \sum_{\{ v \in K \,|\, v \,\,\textit{active} \} } (A_v - b) \Bigr) \,}{\,3\,} \,\biggr)\biggr) \\
\nonumber   &=& \exp \, \biggr( -\, \frac{\, \Bigr(\, S_1(K, n) (b - a) \,+\, \sum_{\{ v \in K \,|\, v \,\,\textit{active} \} } (A_v - b) \Bigr)^2 \,}{\,2\,} \\
\nonumber  & & \quad\quad\quad \,:\, \biggr(\, |\,K \setminus K_0\,| \cdot 2 \sigma^2 \,+\, \frac{\,2 \, M \, \Bigr(\, S_1(K, n) (b - a) \,+\, \sum_{\{ v \in K \,|\, v \,\,\textit{active} \} } (A_v - b) \Bigr)\,}{\,3\,} \,\biggr)\biggr) \\
    &=& \exp \, \biggr( -\, \frac{\, 3\, \Bigr(\, S_1(K, n) (b - a) \,+\, \sum_{\{ v \in K \,|\, v \,\,\textit{active} \} } (A_v - b) \Bigr)^2 \,}{\, 12\, \sigma^2 \,|\,K \setminus K_0\,|\, +\, 4 M \cdot \Bigr(\, S_1(K, n) (b - a) \,+\, \sum_{\{ v \in K \,|\, v \,\,\textit{active} \} } (A_v - b) \Bigr) \,} \biggr) \,.
\end{eqnarray}

\end{proof}

\begin{proof} (Proposition \ref{Proposition_7})
Follows from Proposition \ref{Proposition_Basic_Inequality} by bounding the probability of the union of individual model selection events.
\end{proof}

\begin{proof} (Theorems \ref{Theorem_2} and \ref{Theorem_1})
Applying Propositions \ref{Proposition_5}, \ref{Proposition_7}, \ref{Proposition_Basic_Inequality} and \ref{Bernstein_Inequality} for $\alpha (n) \,=\, \delta \, k(n)$ and $M$ replaced by $2\,M$ directly leads to the consistency statement of Theorem \ref{Theorem_2}. Noticing that in the statement of Theorem \ref{Theorem_1} it holds that $|\mathcal{K}_{\mathcal{A}}(n)| \,=\, n$, we derive Theorem \ref{Theorem_1} from Theorem \ref{Theorem_2}.
\end{proof}

\end{document}